\documentclass{article}

\PassOptionsToPackage{numbers, sort}{natbib}

\usepackage[preprint]{neurips_2026}

\usepackage[utf8]{inputenc} 
\usepackage[T1]{fontenc}    
\usepackage{hyperref}       
\usepackage{url}            
\usepackage{booktabs}       
\usepackage{amsfonts}       
\usepackage{nicefrac}       
\usepackage{microtype}      
\usepackage{xcolor}         
\usepackage{amsmath}
\usepackage{amssymb}
\usepackage{amsthm}
\usepackage{mathtools}
\usepackage{algorithm}
\usepackage{algorithmic}
\usepackage{enumitem}
\newtheorem{proposition}{Proposition}

\usepackage{multirow} 
\usepackage{colortbl}
\usepackage{wrapfig}

\usepackage[most]{tcolorbox}
\usepackage{listings}
\usepackage{setspace}

\definecolor{myblue}{RGB}{0, 100, 100}

\hypersetup{
    colorlinks=true,
    linkcolor=myblue,
    citecolor=myblue,
    urlcolor=myblue
}

\definecolor{promptgray}{RGB}{245,245,245}

\newtcolorbox{promptbox}[1]{
    enhanced,
    breakable,
    colback=white,
    colframe=black,
    colbacktitle=black,
    coltitle=white,
    title=#1,
    fonttitle=\bfseries,
    boxrule=0.8pt,
    arc=2mm,
    left=2mm,
    right=2mm,
    top=2mm,
    bottom=2mm,
    before upper={\setstretch{1.12}}
}

\lstdefinestyle{promptjson}{
    basicstyle=\ttfamily\small,
    breaklines=true,
    columns=fullflexible,
    keepspaces=true,
    showstringspaces=false,
    frame=none
}
\title{Learning to Hint for Reinforcement Learning}

\author{%
  Yu Xia$^1$ \quad Canwen Xu$^2$\thanks{\ To whom correspondence should be addressed. Work done during Yu Xia's internship at Snowflake.} \quad Zhewei Yao$^2$ \quad Julian McAuley$^1$ \quad Yuxiong He$^2$ \\
  $^1$University of California, San Diego \quad $^2$Snowflake AI Research\\
  $^1$\texttt{\{yux078,jmcauley\}@ucsd.edu} \\ $^2$\texttt{\{canwen.xu,zhewei.yao,yuxiong.he\}@snowflake.com}\\
}

\begin{document}

\maketitle

\begin{abstract}
    Group Relative Policy Optimization (GRPO) is widely used for reinforcement learning with verifiable rewards, but it often suffers from {advantage collapse}: when all rollouts in a group receive the same reward, the group yields zero relative advantage and thus no learning signal. For example, if a question is too hard for the reasoner, all sampled rollouts can be incorrect and receive zero reward. Recent work addresses this issue by adding hints or auxiliary scaffolds to such hard questions so that the reasoner produces mixed outcomes and recovers a non-zero update. However, existing hints are usually fixed rather than adapted to the current reasoner, and a hint that creates learning signal under the hinted input does not necessarily improve the no-hint policy used at test time. 
    To this end, we propose \textbf{Hi}nt \textbf{L}earning for Reinforcement \textbf{L}earning (HiLL), a framework that jointly trains a hinter policy and a reasoner policy during RL. For each hard question, the hinter generates hints online conditioned on the current reasoner’s incorrect rollout, allowing hint generation to adapt to the reasoner’s evolving errors. We further introduce \emph{hint reliance}, which measures how strongly correct hinted trajectories depend on the hint. We derive a transferability result showing that lower hint reliance implies stronger transfer from hinted success to no-hint success, and we use this result to define a transfer-weighted reward for training the hinter. Therefore, HiLL favors hints that not only recover informative GRPO groups, but also produce signals that are more likely to improve the original no-hint policy. Experiments across multiple benchmarks show that HiLL consistently outperforms GRPO and prior hint-based baselines, demonstrating the value of adaptive and transfer-aware hint learning for RL. 
    The code is available at \url{https://github.com/Andree-9/HiLL}.
\end{abstract}
\section{Introduction}

Reinforcement learning with verifiable rewards (RLVR) has become a standard approach to improve large language model (LLM) reasoning, especially in domains such as mathematics where final-answer correctness can be checked reliably \citep{guo2025deepseekr1,wen2025rlvrboundary}. 
Among recent methods, Group Relative Policy Optimization (GRPO) \citep{shao2024deepseekmath} is widely used, which removes the value critic and estimates advantages directly from a group of sampled rollouts. 
However, under binary outcome rewards, this design introduces a basic failure mode: if all rollouts in a group receive the same reward, then their relative advantages are all zero and the question produces no policy gradient \citep{mroueh2025grpoeffective,xiong2025reinforceada,le2025zvp,liao2026sage}. 
This issue appears at both ends of the difficulty spectrum. 
Easy questions often yield all-correct groups, while hard questions often yield all-incorrect groups. 
In both cases, no learning signal is obtained and no update is made. Moreover, for hard questions that the current reasoner policy cannot solve at all, standard online RL has no effective way to improve as no correct trajectory is observed and no reward signal is available \citep{qu2026pope,chen2025nurl}. 
Therefore, the questions that matter most for expanding the model’s reasoning ability are often exactly the ones that provide no learning signal.

A growing body of work tries to address this issue. 
One line of work allocates more rollouts to hard questions in order to recover rare successes and reduce variance under a fixed compute budget \citep{yao2025gvmraft,xiong2025reinforceada}.
Another line filters, skips, downsamples, or reshapes uninformative groups to reduce wasted computation \citep{zheng2025greso,xu2025pods,le2025zvp,yu2025dapo}.
A third line changes the input itself by adding hints, privileged prefixes, or reasoning scaffolds to hard questions so as to induce mixed rollout outcomes and recover non-zero GRPO signals \citep{chen2025nurl,zhang2025stephint,liao2026sage,zhang2025scaf}.
These methods are effective in different settings, but they address different aspects of the problem. 
More sampling may recover signal at high cost, while fixed prompt interventions may create signal without being matched to the current reasoner.

In this paper, we focus on the hint-based direction and argue that two questions are central. 
First, can hints adapt to the reasoner’s current failure modes, rather than being fixed in advance? Second, even if a hint creates mixed outcomes and restores a GRPO signal, does that learning signal actually help the no-hint policy used at test time? 
Existing hint-based methods typically rely on partial solutions, handcrafted scaffolds, or offline generated hints \citep{chen2025nurl,zhang2025scaf}, which are fixed rather than tailored to the current reasoner. 
In addition, they mainly value a hint for turning an all-incorrect question into a mixed-outcome group. 
However, not every such hint is useful. A hint may produce correct hinted rollouts simply by making the problem much easier, without teaching the reasoner behavior that remains useful when the hint is removed. 
In other words, creating signal is not enough: the signal should also transfer back to the original question.

We address both issues with our proposed \textbf{Hi}nt \textbf{L}earning for Reinforcement \textbf{L}earning (HiLL) framework, which treats hint generation itself as a learnable objective. HiLL jointly trains a hinter policy together with the reasoner during RL. For each hard question, the hinter generates hints online by conditioning on the question, an incorrect rollout from the current reasoner, and the reference solution. This enables the hinter to adapt to the reasoner’s evolving errors over training, rather than relying on static offline hints that may stop being useful as the reasoner changes.

Moreover, HiLL provides a transfer-aware view of hint quality. 
A good hint should do more than create a non-zero GRPO update: it should guide the reasoner toward correct trajectories that are still plausible under the original no-hint question. If a correct hinted trajectory remains likely even after the hint is removed, then training on that trajectory is more likely to improve the no-hint policy. By contrast, if success depends strongly on the hint, then the hint acts more like a shortcut than a useful teaching signal. To capture this distinction, we introduce \emph{hint reliance}, which measures how much correct hinted trajectories depend on the hint relative to the original question. We then derive a transferability result showing that lower hint reliance implies stronger transfer from hinted success to no-hint success. Based on this result, we define a transfer-weighted reward for the hinter training, encouraging hints that both create informative GRPO groups and produce learning signals that are more likely to transfer and help the reasoner improve in the no-hint scenario.

Therefore, HiLL turns hinting for RL into an online co-training problem between a hinter policy and a reasoner policy. As training progresses, previously hard questions may become solvable, while newly difficult questions define the reasoner’s current capability boundary. 
The hinter is thus trained on an evolving distribution of reasoner failures and learns to adapt its generated hints to the reasoner’s changing weaknesses. 
This makes hinting a dynamic part of the RL process rather than a fixed preprocessing step.
In summary, we make the following contributions:
\begin{itemize}[left=0pt, parsep=0pt, partopsep=0pt]
    \item We propose HiLL, a co-training framework that jointly optimizes a hinter policy and a reasoner policy, enabling online hint generation conditioned on the current reasoner's failures.
    \item We introduce hint reliance as a measure of whether hinted success is likely to transfer back to the original no-hint setting, derive a transferability result based on this quantity, and use it to construct a transfer-weighted reward for training the hinter policy.
    \item We show on comprehensive benchmarks that HiLL consistently outperforms standard GRPO and hint-based baselines, demonstrating the value of adaptive and transfer-aware hint learning for RL.
\end{itemize}
\section{Related Work}

\textbf{Sampling and filtering strategies for RLVR.}
When GRPO rollout groups are uniformly correct or incorrect, within-group advantages collapse and the question yields no gradient~\citep{mroueh2025grpoeffective,xiong2025reinforceada,le2025zvp,liao2026sage}.
Adaptive sampling methods address this by reallocating rollout budgets toward questions near the model's capability boundary, e.g., via variance-minimizing allocation~\citep{yao2025gvmraft}, adaptive importance-based scheduling~\citep{xiong2025reinforceada}, or knapsack-style budget optimization~\citep{li2025knapsack}.
Filtering and reshaping methods take a complementary approach, skipping or clipping degenerate groups~\citep{yu2025dapo,zheng2025greso}, downsampling uninformative rollouts~\citep{xu2025pods}, or extracting signal from zero-variance groups via entropy-guided advantage shaping~\citep{le2025zvp} and virtual reward calibration~\citep{nan2025ngrpo}.
Curriculum-based scheduling further improves sample efficiency by ordering questions from easy to hard or partitioning by difficulty~\citep{parashar2025e2h,zhang2025clpo}.
These strategies are complementary to hinting: they improve how existing signal is used but cannot create signal when the model's success probability on a question is effectively zero.

\textbf{Privileged hinting and scaffolded RL.}
A growing line of work modifies the input during LLM training by injecting privileged information to induce correct rollouts hard questions, including oracle solution prefixes~\citep{qu2026pope}, self-generated hints~\citep{liao2026sage}, external teacher hints~\citep{zhang2025scaf}, multi-level progressive hints~\citep{zhang2025stephint}, curriculum-based scaffolds~\citep{chen2025nurl}, or initial steps extracted from rare within-batch successes~\citep{deng2026hipo}.
A key design principle shared across these methods is that the model should be trained on-policy with respect to the augmented input, i.e., both rollout sampling and the policy-gradient loss should condition on the privileged context, not the original question alone~\citep{liao2026sage,qu2026pope}.
Recent work also utilizes such privileged information for on-policy self-distillation, where a model generates on-policy trajectories under a student context and receives dense supervision from its own privileged context conditioned on ground-truth solutions, verified reasoning traces, or rich textual feedback~\citep{zhao2026opsd,hubotter2026sdpo,shenfeld2026self}.
Our HiLL framework shares this on-policy structure, where privileged information reshapes on-policy trajectories during training and is removed at test time, but differs in two key respects: (i)~it actively {learns} to generate better hints through RL during the privileged training rather than relying on fixed or pre-collected contexts, and (ii)~it explicitly optimizes for {transferability} by penalizing hints whose induced correct trajectories are unlikely to improve the model under no-hint scenarios.

\section{Preliminaries}
\label{sec:prelim}

\begin{figure*}[t!]
    \centering
    \includegraphics[width=\textwidth]{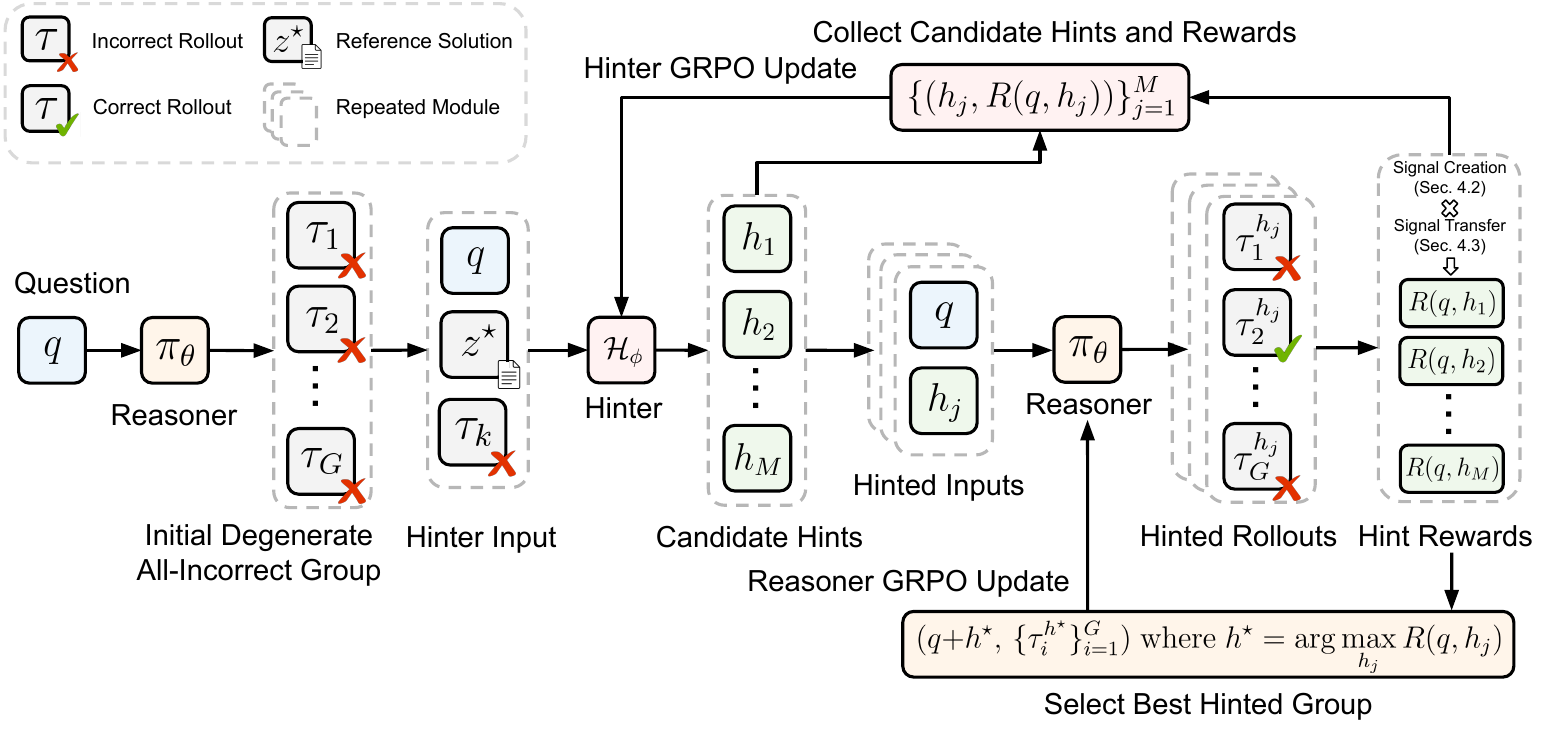}
    \vspace{-1em}
    \caption{{Overview of our HiLL framework.}
    Given a question~$q$ with an all-incorrect group, the hinter~$\mathcal{H}_\phi$ takes the question, a failed rollout~$\tau_k$, and the reference solution~$z^\star$ as input, and generates $M$ candidate hints.
    The reasoner~$\pi_\theta$ re-samples $G$ rollouts under each hinted input $q{+}h_j$.
    Each hint is then scored by {signal creation} (Sec.~\ref{subsec:signal_creation}) and {signal transfer} (Sec.~\ref{subsec:transfer}).
    The best hinted group is selected reasoner GRPO update, while all candidate hints form the group for the hinter GRPO update.}
    \label{fig:overview}
    \vspace{-0.35em}
\end{figure*}

Consider reasoning with verifiable rewards.
Given a question $q$, a reasoner policy $\pi_\theta$ samples a trajectory $\tau\sim\pi_\theta(\cdot\mid q)$, and a verifier returns a binary reward $r(\tau)\in\{0,1\}$ indicating whether the final answer is correct.
We denote the per-question success probability by $p_\theta(q)=\Pr_{\tau\sim\pi_\theta(\cdot\mid q)}[r(\tau)=1]$.

\textbf{GRPO and advantage collapse.}
Group Relative Policy Optimization~\citep{shao2024deepseekmath} samples a group of $G$ trajectories for the same question and computes within-group normalized advantages:
\begin{equation}
    A_i \;=\; \frac{r_i - \bar{r}}{\operatorname{std}(r_{1:G})+\epsilon},
    \qquad
    \bar{r} = \frac{1}{G}\sum_{i=1}^{G} r_i.
    \label{eq:grpo_adv}
\end{equation}
Under binary rewards, if all $G$ rollouts receive the same reward, every advantage vanishes and the question produces no gradient.
For a question with per-rollout success probability~$p$, the probability that a group of size~$G$ contains both correct and incorrect outcomes is
\begin{equation}
    s(p;\,G) \;=\; 1 - p^{G} - (1-p)^{G}.
    \label{eq:nondeg_prob}
\end{equation}
This \emph{non-degenerate probability} is maximized at $p=\tfrac{1}{2}$ and satisfies $s(p;G)\approx Gp$ when $p\approx 0$.
Hard questions with small success probability are therefore most likely to yield degenerate groups and produce no learning signal.


\section{HiLL: Hint Learning for Reinforcement Learning}
\label{sec:method}

HiLL jointly trains a hinter policy alongside the reasoner to recover learning signal from hard questions.
It intervenes only on all-incorrect GRPO groups, which produce no gradient under binary rewards.
As illustrated in Figure~\ref{fig:overview} and summarized in Algorithm~\ref{alg:hill}, the hinter generates candidate hints online, the reasoner re-samples under each hinted input, and hints are scored by both signal creation and signal transfer before the best hinted group replaces the original degenerate group.

\subsection{Failure-Conditioned Hint Generation}
\label{subsec:hint_generation}

Given a training batch $\mathcal{B}$, the reasoner first samples $G$ rollouts per question and identifies the all-incorrect subset $\mathcal{I}=\{q\in\mathcal{B}:\sum_{i=1}^{G}r_i(q)=0\}$.
For each $q\in\mathcal{I}$, the hinter policy $\mathcal{H}_\phi$ generates $M$ candidate   hints conditioned on the question~$q$, a randomly selected incorrect rollout~$\tau_k$ from the current reasoner's all-incorrect group, and a reference solution~$z^\star$ available only during training:
\begin{equation}
    \{h_j\}_{j=1}^M \;\sim\; \mathcal{H}_\phi(\cdot \mid q,\, \tau_k,\, z^\star).
\end{equation}
The incorrect rollout exposes the reasoner's current error mode, while the reference solution provides supervision for producing a targeted hint.
We instruct the hinter to analyze the reasoner's failure and output a concise pedagogical hint where the prompt is provided in Appendix~\ref{app:hint_generator_prompt}.

A candidate hint is invalid if it (i)~violates the required output format, (ii)~leaks the final answer or key intermediate computations, or (iii)~causes the concatenated input $q{+}h$ to exceed the maximum context length.
Invalid hints receive a fixed failure penalty $R_{\mathrm{fail}}<0$ as defined in Section~\ref{subsec:reward}.

\subsection{Hinted Rollouts and Signal Creation}
\label{subsec:signal_creation}

For each valid candidate hint~$h_j$, the reasoner re-samples $G$ rollouts from $\pi_\theta(\cdot\mid q{+}h_j)$ and computes rewards.
Let $\hat{p}_h = \frac{1}{G}\sum_{i=1}^{G} r_i^{(h_j)}$ denote the estimated success rate under the hinted input.
A hint creates useful signal when the resulting group is non-degenerate, i.e., contains both correct and incorrect rollouts.
We quantify this with the non-degenerate probability $s(\hat{p}_h;\,G) = 1 - \hat{p}_h^G - (1{-}\hat{p}_h)^G$ from Eq.~\eqref{eq:nondeg_prob}, which vanishes when $\hat{p}_h\in\{0,1\}$ and peaks at $\hat{p}_h=\tfrac{1}{2}$.

However, signal creation alone is not sufficient: a hint may produce mixed outcomes simply by making the problem much easier, without producing trajectories that remain useful once the hint is removed.
We therefore introduce a second criterion that measures whether the created signal transfers to the original no-hint question.

\begin{algorithm}[!t]
\caption{HiLL: Hint Learning for Reinforcement Learning}
\label{alg:hill}
\begin{algorithmic}[1]
\REQUIRE Batch $\mathcal{B}$, reasoner $\pi_\theta$, hinter
$\mathcal{H}_\phi$, group size $G$, hint candidates $M$, temperature $T$,
failure reward $R_{\mathrm{fail}}$
\STATE For each $q\in\mathcal{B}$, sample $G$ rollouts from
       $\pi_\theta(\cdot\mid q)$ and compute rewards
\STATE Identify all-incorrect questions
       $\mathcal{I}=\{q\in\mathcal{B}:\sum_{i=1}^G r_i(q)=0\}$
\FOR{each $q\in\mathcal{I}$}
    \STATE Form hinter input $c=(q,\,\tau_k,\,z^\star)$ using an incorrect
           rollout and reference solution
    \STATE Sample $M$ candidate hints
           $\{h_j\}_{j=1}^M\sim\mathcal{H}_\phi(\cdot\mid c)$
    \FOR{each candidate hint $h_j$}
        \IF{$h_j$ is invalid}
            \STATE Set $R(q,h_j)\leftarrow R_{\mathrm{fail}}$
        \ELSE
            \STATE Sample $G$ rollouts from $\pi_\theta(\cdot\mid q{+}h_j)$
                   and compute rewards
            \STATE Estimate
                   $\hat{p}_h=\frac{1}{G}\sum_{i=1}^G r_i^{(h_j)}$
            \IF{$0<\hat{p}_h<1$}
                \STATE Estimate $\hat\rho_c(q,h_j)$ from correct hinted
                       trajectories via Eq.~\eqref{eq:est_reliance}
            \ENDIF
            \STATE Compute $R(q,h_j)$ using Eq.~\eqref{eq:hill_reward}
        \ENDIF
    \ENDFOR
    \STATE Select $h^\star=\arg\max_{h_j}R(q,h_j)$
    \IF{$R(q,h^\star)>0$}
        \STATE $(q,\,\{\tau_i\}_{i=1}^G) \gets (q{+}h^\star,\,\{\tau_i^{h^\star}\}_{i=1}^G)$
    \ENDIF
\ENDFOR
\STATE Update reasoner $\pi_\theta$ using Eq.~\eqref{eq:reasoner} on the
       final selected groups
\STATE Update hinter $\mathcal{H}_\phi$ using Eq.~\eqref{eq:hinter}, where
       $M$ hints per $q\in\mathcal{I}$ form the group
\end{algorithmic}
\end{algorithm}

\subsection{Hint Reliance and Signal Transfer}
\label{subsec:transfer}

If correct hinted trajectories depend strongly on the hint, training on them may teach the reasoner to exploit the hint rather than improve its no-hint reasoning capability.
For example, if the reasoner fails to simplify a key expression, a hint like ``try factoring the left-hand side'' suggests a strategy the model could have found on its own, whereas ``note that $x^2{-}5x{+}6=(x{-}2)(x{-}3)$'' performs the critical step directly for the reasoner.
Correct trajectories from the latter depend largely on the hint and are unlikely to transfer.

\textbf{Hint reliance.}
To formalize this, we define the \emph{hint reliance} of a trajectory~$\tau$ sampled under $q{+}h$ as
\begin{equation}
    \rho(\tau;\, q, h)
    \;=\;
    \log\pi_\theta(\tau \mid q{+}h)
    \;-\;
    \log\pi_\theta(\tau \mid q).
    \label{eq:hint_reliance}
\end{equation}
A value near zero means $\tau$ is roughly equally likely with or without the hint, indicating low dependence. A large positive value means $\tau$ is much more likely under the hinted input, indicating strong reliance.
We average over correct hinted trajectories, which carry the informative GRPO signal:
\begin{equation}
    \rho_c(q,h)
    \;=\;
    \frac{1}{|\mathcal{C}|}
    \sum_{\tau\in\mathcal{C}} \rho(\tau;\,q,h),
    \qquad
    \mathcal{C}
    = \bigl\{\tau\sim\pi_\theta(\cdot\mid q{+}h) : r(\tau)=1\bigr\}.
    \label{eq:avg_reliance}
\end{equation}

The following result shows that this average hint reliance directly controls how well hinted success transfers to the no-hint setting.
\begin{proposition}[Transfer bound]
\label{prop:transfer}
For a question-hint pair $(q,h)$, let $P_h(\tau)=\pi_\theta(\tau\mid q{+}h)$ and $P(\tau)=\pi_\theta(\tau\mid q)$ denote the rollout distributions under the hinted and original inputs, with success probabilities $p_h=P_h\bigl(r(\tau){=}1\bigr)$ and $p=P\bigl(r(\tau){=}1\bigr)$.
If $p_h>0$, then
\begin{equation*}
    \rho_c(q,h)
    \;=\;
    \log\frac{p_h}{p}
    \;+\;
    D_{\mathrm{KL}}\!\bigl(
        P_h(\cdot\mid r{=}1)
        \;\big\|\;
        P(\cdot\mid r{=}1)
    \bigr),
    \label{eq:transfer_identity}
\end{equation*}
and therefore
\begin{equation*}
    p
    \;\ge\;
    p_h \cdot \exp\bigl(-\rho_c(q,h)\bigr).
    \label{eq:transfer_bound}
\end{equation*}
\end{proposition}

The proof is provided in Appendix~\ref{app:proof_transfer}. The identity decomposes hint reliance into two terms: the log-ratio of success probabilities and a KL divergence measuring how the distribution over correct trajectories shifts when the hint is added. Since the KL term is non-negative, the bound follows directly: the no-hint success probability $p$ is at least the hinted success probability $p_h$ discounted by $\exp(-\rho_c)$. Therefore, lower hint reliance yields a tighter bound and a stronger transfer guarantee.

\textbf{Practical estimator.}
In practice, we normalize $\rho(\tau;\,q,h)$ by the trajectory length~$|\tau|$ to reduce length bias across trajectories of varying lengths:
\begin{equation}
    \hat\rho_c(q,h)
    \;=\;
    \frac{1}{|\mathcal{C}|}
    \sum_{\tau\in\mathcal{C}}
    \frac{\rho(\tau;\,q,h)}{|\tau|}.
    \label{eq:est_reliance}
\end{equation}
Computing $\hat\rho_c$ requires scoring each correct hinted trajectory under both $q{+}h$ and~$q$, amounting to two teacher-forced forward passes of~$\pi_\theta$.

\subsection{Transfer-Weighted Hinter Reward}
\label{subsec:reward}

Building on the previous signal creation and transfer measure, we now define the reward for each candidate hint for hinter training:
\begin{equation}
    R(q,h)
    \;=\;
    \begin{cases}
        R_{\mathrm{fail}},
            & \text{if } h \text{ is invalid},
            \\[4pt]
        \displaystyle
        \underbrace{s(\hat{p}_h;\,G)\vphantom{\exp\biggl(
            -\frac{\max\bigl(\hat\rho_c(q,h),\;0\bigr)}{T}
        \biggr)}}_{\text{signal creation}}
        \;\cdot\;
        \underbrace{\exp\biggl(
            -\frac{\max\bigl(\hat\rho_c(q,h),\;0\bigr)}{T}
        \biggr)}_{\text{signal transfer}},
            & \text{otherwise},
    \end{cases}
    \label{eq:hill_reward}
\end{equation}
where $T>0$ is a temperature parameter.
The first term $s(\hat{p}_h;\,G)$ from Eq.~\eqref{eq:nondeg_prob} rewards hints that produce mixed-outcome groups and vanishes when all rollouts are correct or all incorrect.
The second term based on Eq.~\eqref{eq:est_reliance} serves as a transfer weight: it equals~$1$ when hint reliance is zero or negative; it decays as reliance increases.
Negative reliance means the correct trajectory is already at least as likely under the original question, so we treat it as fully transferable and do not further boost the reward.
Their product encourages hints that both recover non-zero GRPO signal and produce trajectories that transfer to the no-hint setting.
When $\hat{p}_h\in\{0,1\}$, signal creation is zero, so no hint reliance computation is needed.

\subsection{Joint Reasoner-Hinter Optimization}
\label{subsec:training}

After scoring all candidates, the best hint $h^\star=\arg\max_j R(q,h_j)$ is selected for each $q\in\mathcal{I}$.
If $R(q,h^\star)>0$, the batch entry is updated via replacement
$(q,\,\{\tau_i\}_{i=1}^G) \gets (q{+}h^\star,\,\{\tau_i^{h^\star}\}_{i=1}^G)$,
where $\tau_i^{h^\star}\!\sim\pi_\theta(\cdot\mid q{+}h^\star)$. Otherwise, the original entry is kept.

\textbf{Reasoner update.}
After replacement, each question $q\in\mathcal{B}$ is paired with an input $x_q$ (either $q$ or $q{+}h^\star$).
Note that when a hint is used, both the rollout sampling and the log-probability computation condition on the same hinted input $x_q = q{+}h^\star$, keeping the reasoner update on-policy with respect to the augmented context.
The reasoner is then updated with the standard GRPO objective:
\begin{equation}
    \mathcal{L}_{\mathrm{R}}(\theta)
    \;=\;
    -\mathbb{E}_{q\sim\mathcal{B}}
    \Biggl[\,
    \sum_{i=1}^{G}
    A_i
    \sum_{t=1}^{|\tau_i|}
    \log\pi_\theta\!\bigl(y_{i,t}\mid x_q,\,y_{i,<t}\bigr)
    \Biggr],
    \label{eq:reasoner}
\end{equation}
where $\tau_i\sim\pi_\theta(\cdot\mid x_q)$ and $A_i$ are within-group normalized advantages as in Eq.~\eqref{eq:grpo_adv}.

\textbf{Hinter update.}
For each $q\in\mathcal{I}$, the $M$ candidate hints form a GRPO group with rewards $\{R(q,h_j)\}_{j=1}^M$, and the hinter is updated via:
\begin{equation}
    \mathcal{L}_{\mathrm{H}}(\phi)
    \;=\;
    -\mathbb{E}_{q\sim\mathcal{I}}
    \Biggl[\,
    \sum_{j=1}^{M}
    A_j
    \sum_{t=1}^{|h_j|}
    \log\mathcal{H}_\phi(h_{j,t}\mid c_q,\,h_{j,<t})
    \Biggr],
    \label{eq:hinter}
\end{equation}
where $A_j$ are group-normalized advantages from the hint rewards, and $c_q=(q,\tau_k,z^\star)$.
Invalid candidates participate through $R_{\mathrm{fail}}$, allowing the hinter to learn to avoid producing them.

\textbf{Co-training dynamics.}
Because both policies are updated jointly, the hinter adapts as the reasoner improves.
Questions that become solvable leave~$\mathcal{I}$, while the remaining failures reflect the reasoner's current capability frontier.
The hinter is thus trained on an evolving distribution of reasoner failures and learns to adjust its hints to the reasoner's changing weaknesses, rather than relying on a fixed hinting strategy throughout training.

\section{Experiments}

\subsection{Experimental Setup}
\label{subsec:setup}

\textbf{Models.}
We evaluate our HiLL framework on two models as the reasoner policy: Llama-3.2-3B-Instruct~\citep{meta2024llama3} and Qwen2.5-7B-Instruct~\citep{yang2024qwen25}.
The hinter policy is initialized from Qwen3-4B-Instruct~\citep{yang2025qwen3} for all experiments.
Both policies are trained jointly via GRPO as described in Section~\ref{sec:method}.

\textbf{Training data.}
We use the same 15k-prompt subset of OpenR1-Math-220k~\citep{huggingface2025openr1} curated by \citet{liao2026sage}, drawn from NuminaMath~1.5~\citep{li2024numinamath} together with ground-truth answers and reference solutions.
No pass-rate filtering is applied, so questions span a wide difficulty range.

\textbf{Evaluation.}
We report Average@16 accuracy on six math reasoning benchmarks: AIME24~\citep{maa2024aime}, AIME25~\citep{maa2025aime}, AMC23~\citep{li2024numinamath}, MATH-500~\citep{hendrycks2021math}, Minerva Math~\citep{lewkowycz2022minerva}, OlympiadBench~\citep{he2024olympiadbench}, and two non-math benchmarks: GPQA-diamond~\citep{rein2024gpqa}, MMLU-Pro~\citep{wang2024mmlupro} for generalization assessment.
All reasoner models are evaluated with temperature~$0.6$, top-$p$~$0.95$, and maximum response length~$8{,}192$.
The hinter is never used at evaluation.

\textbf{Baselines.}
We compare against the following baselines trained on the same 15k prompts:
(1)~Base, the initial reasoner;
(2)~GRPO~\citep{shao2024deepseekmath}, standard RL without hints;
(3)~LUFFY~\citep{yan2025luffy}, which replaces one on-policy rollout with an off-policy trajectory from DeepSeek-R1;
(4)~Scaf-GRPO~\citep{zhang2025scaf}, which augments all-incorrect groups with hints from an external teacher model;
and (5)~SAGE~\citep{liao2026sage}, on-policy hinting with self-generated hints conditioned on reference solutions.
The implementation builds on the official codebase of \citet{liao2026sage}, and we report baseline results from their paper.
We additionally report HiLL$_{\text{w/o TW}}$, an ablated variant of HiLL that removes the transfer weight from the hinter reward (Eq.~\eqref{eq:hill_reward}), using only the non-degenerate probability as the reward for hinter training.

\textbf{Implementation details.}
We use verl~\citep{sheng2025verl} for RL training and vLLM~\citep{kwon2023vllm} for rollout generation and all experiments run on $8{\times}$B200 GPUs.
For reasoner policy training in all methods, we follow DAPO~\citep{yu2025dapo} to disable the KL penalty and apply Clip-Higher with $\epsilon_{\mathrm{low}}{=}0.2$, $\epsilon_{\mathrm{high}}{=}0.28$.
We train the reasoner for 500 steps with a batch size of 128, $G{=}8$ rollouts per prompt, and learning rate $10^{-6}$.
The maximum prompt and response lengths are set to 2{,}048 and 8{,}192 tokens.
We evaluate every 50 steps and select the reasoner checkpoint with the best Average@16 accuracy to report the results.
For hinter policy training in HiLL, the hinter generates $M{=}4$ candidate hints per all-incorrect question with a maximum prompt and response length being 10{,}240 and 1{,}024 tokens, transfer temperature $T{=}0.3$, failure penalty $R_{\mathrm{fail}}{=}{-}0.2$, and the same learning rate and clipping as the reasoner.
We use Ray~\citep{moritz2018ray} to co-locate reasoner and hinter models on the same GPU nodes. Since the co-training pipeline executes each phase sequentially as in Algorithm~\ref{alg:hill}, Ray idles one policy while the other is active, letting both FSDP-sharded models share the same devices without extra memory overhead compared to baseline training methods.
Prompts used for reasoner and hinter are provided in Appendix~\ref{app:hint_generator_prompt}.

\begin{table}[!t]
    \caption{Main results on in-distribution and out-of-distribution benchmarks. Best results are in {bold} and second-best are {underlined}. HiLL consistenctly outperforms GRPO and hint-based baselines.}
    \vspace{-1.25em}
    \label{tab:main_results}
    \begin{center}
    \resizebox{\linewidth}{!}{%
    \begin{tabular}{l|cccccc|ccc}
    \toprule
        \multirow{2}{*}{\textbf{Method}} & \multicolumn{6}{c}{\textbf{In-distribution}} & \multicolumn{3}{c}{\textbf{Out-of-distribution}} \\
        \cmidrule(lr){2-7} \cmidrule(lr){8-10}
        & \textbf{AIME24 / 25} & \textbf{AMC23} & \textbf{MATH-500} & \textbf{Minerva} & \textbf{Olympiad} & \textbf{Avg.} & \textbf{GPQA} & \textbf{MMLU-Pro} & \textbf{Avg.} \\
    \midrule

        \multicolumn{10}{l}{\textit{Llama-3.2-3B-Instruct}} \\
        \cmidrule(lr){1-10}
        Base & 6.5 / 0.6 & 22.8 & 44.7 & 17.8 & 14.2 & 17.8 & 17.9 & 27.0 & 22.5 \\
        GRPO & 6.7 / 0.8 & 29.5 & 52.1 & 20.5 & 21.8 & 21.9 & 26.3 & 39.8 & 33.1 \\
        LUFFY & 4.4 / 0.4 & 18.6 & 38.9 & 14.3 & 11.9 & 14.7 & 16.0 & 26.7 & 21.4 \\
        Scaf-GRPO & 7.7 / \textbf{2.3} & 28.8 & 51.7 & 19.4 & 19.5 & 21.5 & 24.1 & 38.0 & 31.0 \\
        {SAGE} & \textbf{9.2} / 0.8 & \underline{34.7} & 56.3 & 20.1 & 22.0 & \underline{23.9} & 27.3 & 40.7 & 34.0 \\
        \rowcolor[RGB]{242, 250, 246}{HiLL}$_\text{w/o TW}$ & 6.9 / \textbf{2.3} & 32.5 & \underline{56.7} & \underline{21.6} & \underline{22.4} & 23.7 & \textbf{28.0} & \underline{41.0} & \underline{34.5} \\
        \rowcolor[RGB]{228, 247, 237}{HiLL} & \underline{8.5} / \underline{1.7} & \textbf{34.8} & \textbf{57.7} & \textbf{21.7} & \textbf{23.2} & \textbf{24.6} & \underline{27.8} & \textbf{42.7} & \textbf{35.3} \\
    \midrule

        \multicolumn{10}{l}{\textit{Qwen2.5-7B-Instruct}} \\
        \cmidrule(lr){1-10}
        Base & 13.8 / 6.7 & 53.4 & 75.7 & 38.1 & 39.2 & 37.8 & 37.1 & 56.4 & 46.7 \\
        GRPO & 15.0 / 13.5 & 55.5 & 79.2 & 39.1 & 44.5 & 41.1 & 37.2 & 57.6 & 47.4 \\
        LUFFY & \textbf{17.1} / 13.5 & 55.2 & \underline{81.3} & 39.0 & 44.2 & 41.7 & 38.1 & 59.1 & 48.6 \\
        Scaf-GRPO & 14.6 / 12.7 & 58.8 & 78.0 & \underline{39.8} & 42.0 & 41.0 & 36.6 & 58.4 & 47.5 \\
        {SAGE} & 16.0 / 12.5 & 60.3 & 80.0 & 39.3 & \underline{45.9} & 42.3 & 38.0 & 59.3 & 48.6 \\
        \rowcolor[RGB]{242, 250, 246}{HiLL}$_\text{w/o TW}$ & 15.2 / \underline{14.4} & \underline{60.5} & 80.5 & \underline{39.8} & 45.7 & \underline{42.7} & \underline{38.6} & \underline{59.9} & \underline{49.2} \\
        \rowcolor[RGB]{228, 247, 237}{HiLL} & \underline{16.9} / \textbf{15.3} & \textbf{63.0} & \textbf{81.8} & \textbf{40.1} & \textbf{48.1} & \textbf{44.2} & \textbf{40.4} & \textbf{61.6} & \textbf{51.0} \\
    \bottomrule
    \end{tabular}%
    }
    \end{center}
    \vspace{-0.75em}
\end{table}

\subsection{Results and Analysis}

\textbf{HiLL consistently boosts reasoner performance compared to baselines.}
Table~\ref{tab:main_results} summarizes the main results. HiLL achieves the highest average accuracy on in-distribution math reasoning tasks on both Llama-3.2-3B-Instruct and Qwen2.5-7B-Instruct. The gains extend to out-of-distribution GPQA and MMLU-Pro despite training exclusively on math data, suggesting that the recovered learning signal in our HiLL framework also generalizes to broader reasoning capabilities. Among hint-based baselines, both Scaf-GRPO and SAGE improve over standard GRPO, confirming that hinting can recover useful signal from hard questions. HiLL outperforms both, indicating that learning to generate adaptive, transfer-aware hints is more effective than relying on fixed external hints or self-generated hints without transfer optimization. LUFFY underperforms on Llama-3.2-3B-Instruct, likely because injecting off-policy DeepSeek-R1 trajectories into a smaller model introduces a distribution mismatch that outweighs the benefit of additional correct rollouts.

\begin{figure}[t!]
    \centering
    \includegraphics[width=0.255\textwidth]{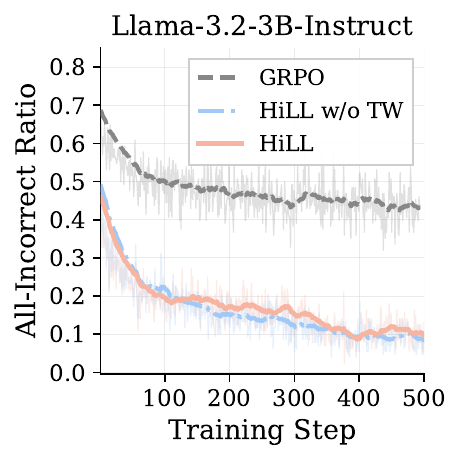}%
    \hspace{-0.5em}
    \includegraphics[width=0.254\textwidth]{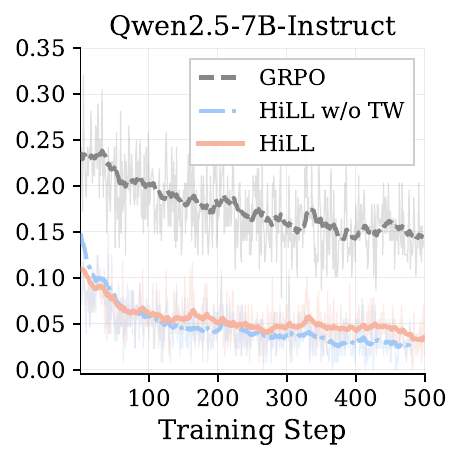}%
    \hspace{-0.5em}
    \includegraphics[width=0.255\textwidth]{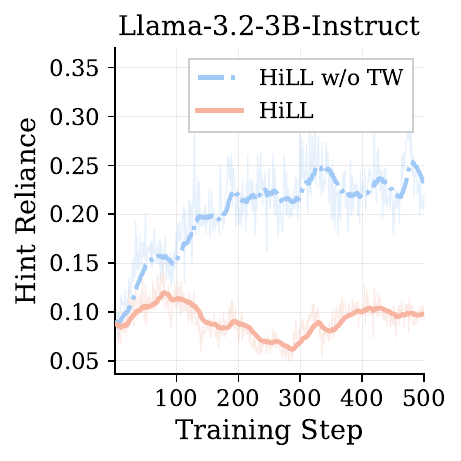}%
    \hspace{-0.5em}
    \includegraphics[width=0.254\textwidth]{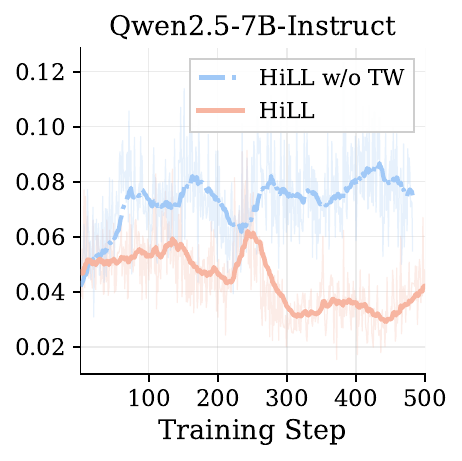}
    \vspace{-1.5em}
    \caption{All-incorrect ratio (left two) and hint reliance (right two) over training steps. Both HiLL variants substantially reduce the fraction of degenerate all-incorrect groups compared to GRPO. Transfer weighting (HiLL vs.\ HiLL$_\text{w/o TW}$) keeps hint reliance consistently lower throughout training.}
    \label{fig:training_dynamics}
    \vspace{-0.25em}
\end{figure}

\textbf{Transfer weighting keeps hint reliance low while recovering GRPO signals.}
Figure~\ref{fig:training_dynamics} tracks the all-incorrect ratio, i.e., the fraction of degenerate all-incorrect groups in the reasoner batch, and its hint reliance over training. Both HiLL variants substantially reduce this ratio compared to standard GRPO, confirming that learned hinting effectively recovers signal from hard questions. The key difference between the two variants appears in the hint reliance curves: without transfer weighting, the hinter is rewarded purely for creating mixed-outcome groups, and reasoner hint reliance climbs steadily as training progresses. With transfer weighting, HiLL keeps hint reliance consistently low, meaning the hinter learns to produce hints whose induced correct trajectories remain plausible under the original no-hint input. With the accuracy gains of HiLL over HiLL$_\text{w/o TW}$ in Table~\ref{tab:main_results}, this confirms that lower hint reliance translates into stronger transfer to the no-hint policy, as predicted by Proposition~\ref{prop:transfer}.

\begin{figure}[t!]
    \centering
    \vspace{-0.75em}
    \includegraphics[width=0.95\linewidth]{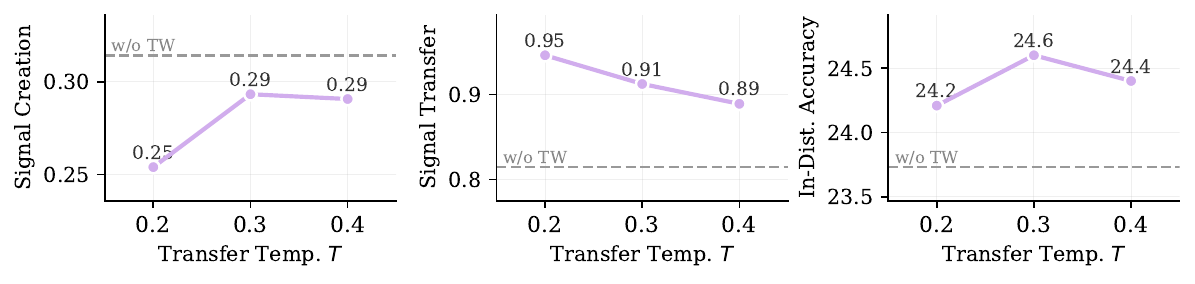}
    \vspace{-0.5em}
    \caption{Effect of transfer temperature $T$ on average signal creation, signal transfer, and in-distribution accuracy for HiLL with Llama-3.2-3B-Instruct. Dashed lines show HiLL$_\text{w/o TW}$.}
    \label{fig:transfer_temp}
    \vspace{-0.75em}
\end{figure}

\textbf{Transfer temperature controls the signal creation--transfer trade-off.}
The temperature $T$ in hinter reward defined in Eq.~\eqref{eq:hill_reward} controls how aggressively the transfer weight penalizes hint reliance. We show in Figure~\ref{fig:transfer_temp} its effect on Llama-3.2-3B-Instruct across three metrics: (i) signal creation, i.e., the average reduction in all-incorrect ratio after hinting where higher value means more degenerate groups in reasoner batch are converted to mixed-outcome groups, (ii) signal transfer, i.e., $\exp(-\hat{\rho}_c)$ averaged over training where higher value means correct hinted trajectories are more likely under the original no-hint input, and (iii) in-distribution accuracy, i.e., average reasoner performance over in-distribution math benchmarks. 
The results show that a smaller $T$ imposes a stronger penalty, yielding higher signal transfer but reducing signal creation as the hinter becomes over-constrained. A larger $T$ relaxes the penalty, creating more signal but with lower transferability. The choice of $T{=}0.3$ in our experiments provides a good balance between signal creation and signal transfer, achieving strong accuracy performance. Nevertheless, all three $T$ values tested in HiLL outperform HiLL$_\text{w/o TW}$, indicating that even a loose transfer weighting is beneficial for the reasoner training.

\begin{table}[t!]
\caption{Example hints generated by HiLL and HiLL$_\text{w/o TW}$ on the same question. HiLL provides conceptual strategies while HiLL$_\text{w/o TW}$ tends to set up the computation directly.}
\vspace{-1em}
\label{tab:hint_examples}
\begin{center}
\fontsize{8.5pt}{10pt}\selectfont
\resizebox{\linewidth}{!}{%
\begin{tabular}{p{0.24\linewidth}|p{0.38\linewidth}|p{0.38\linewidth}}
\toprule
\textbf{Question (excerpt)} & \textbf{Hint from HiLL} & \textbf{Hint from HiLL$_\text{w/o TW}$} \\
\midrule
A line segment $AB$ of length 1 moves on the $x$-axis. Point $P(0,1)$ is connected to $A$ by a line, and point $Q(1,2)$ is connected to $B$ by a line. Find the equation of the trajectory of the intersection point $R$.
& Consider expressing the intersection in terms of a parameter $a$, then eliminate it to reveal the hidden trajectory. \newline \textcolor[RGB]{113, 191, 107}{\scriptsize \textbf{(18 words, hinted pass rate 0.13)}}
& Let $A=(a,0)$ and $B=(a{+}1,0)$ for a real parameter $a\neq 0$. Write the equation of line $PA$, where $P=(0,1)$, and the equation of line $QB$, where $Q=(1,2)$, both in terms of $a$. Find the coordinates $(x,y)$ of their intersection point $R$ as a function of $a$.\ldots \newline \textcolor[RGB]{134, 150, 132}{\scriptsize \textbf{(108 words, hinted pass rate 0.38)}} \\
\midrule
In trapezoid $ABCD$, $AB \parallel CD$, $AB{=}33$, $CD{=}21$, the non-parallel sides are 10 and 14, and $\angle A$, $\angle B$ are both acute. Find the length of the shorter diagonal.
& The shorter diagonal emerges from a triangle where 14, 12, and 10 form a critical configuration, with the angle between them yielding a result just above 25. \newline \textcolor[RGB]{113, 191, 107}{\scriptsize \textbf{(27 words, hinted pass rate 0.13)}}
& Draw a line $DE\parallel BC$ from point $D$ to side $AB$, intersecting $AB$ at point $E$. This creates triangle $ADE$ with $DE=10$, $AE=12$, and $AD=14$. Use the Law of Cosines in triangle $ADE$ to find $\cos A$. Then apply the same angle $\angle A$ in \ldots \newline \textcolor[RGB]{134, 150, 132}{\scriptsize \textbf{(97 words, hinted pass rate 0.50)}} \\
\bottomrule
\end{tabular}%
}
\end{center}
\vspace{-1em}
\end{table}

\textbf{HiLL learns to hint more concisely and conceptually.}
Beyond the quantitative metrics above, transfer weighting also shapes the nature of the generated hints. We report in Figure~\ref{fig:hint_bars} the average hint length in words and the average number of math expressions (e.g., inline equations, symbols) per hint, both computed over hints collected on a fixed interval of 50 steps across training. 
The results show that HiLL produces shorter hints with fewer math expressions than HiLL$_\text{w/o TW}$ on both backbones.
Table~\ref{tab:hint_examples} illustrates the qualitative difference with two actual examples: 
HiLL hints suggest a 
\begin{wrapfigure}{r}{0.3\textwidth}
    \centering
    \vspace{-0.5em}
    \includegraphics[width=0.3\textwidth]{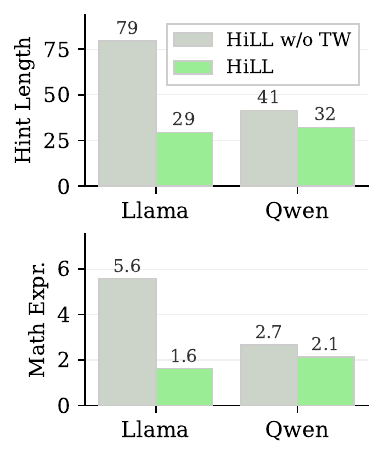}
    \vspace{-2em}
    \caption{Hint length and math expressions per hint.}
    \label{fig:hint_bars}
    \vspace{-1em}
\end{wrapfigure}
high-level strategy (e.g., ``parameterize then eliminate'' or ``look for a critical triangle configuration'') without carrying out the computation, whereas HiLL$_\text{w/o TW}$ hints tend to set up the algebra directly by defining coordinates, writing equations, or applying specific theorems. 
This behavior emerges naturally from our transfer-weighted hint learning objective. Hints that perform key steps for the reasoner would induce high hint reliance, as the resulting correct trajectories become unlikely under the original input. The transfer weight penalizes such hints (Eq.~\eqref{eq:hill_reward}), steering the hinter toward generating more conceptual guidance whose induced correct trajectories the reasoner could plausibly produce on its own.

\textbf{HiLL trades off latency for more effective learning signals.}
HiLL introduces additional per-step computation beyond standard GRPO: for each degenerate all-incorrect group in the batch, the hinter generates a hint, the reasoner re-samples trajectories conditioned on the hinted input, and hint reliance is estimated from the resulting trajectories. The hinter is then updated via its own GRPO step. Because these operations are triggered only for all-incorrect groups, the overhead scales directly with the frequency of such groups. On Llama-3.2-3B-Instruct, which has a high all-incorrect ratio throughout training as shown in Figure~\ref{fig:training_dynamics}, HiLL averages roughly $3.8\times$ the per-step wall-clock time of GRPO. On the stronger Qwen2.5-7B-Instruct backbone, which produces fewer degenerate groups, the multiplier drops to roughly $2.6\times$, comparable to SAGE's $2.3\times$ on the same backbone, as SAGE also targets collapsed groups by progressively trying up to $L{=}3$ hint levels and rolling out at each level until a correct trajectory is found. We view this as a practical trade-off: the extra computation targets exactly the groups from which GRPO extracts no learning signal, converting them into useful training data that drives the accuracy gains reported in Table~\ref{tab:main_results}.

\section{Conclusion}

In this paper, we introduced HiLL, a co-training framework that addresses two limitations of existing hint-based approaches to GRPO's advantage collapse.
First, rather than relying on fixed or externally generated hints, HiLL jointly trains a hinter policy that generates hints online conditioned on the reasoner's current failures, allowing hint generation to stay calibrated as the reasoner improves.
Second, rather than rewarding hints solely for creating mixed-outcome groups, HiLL introduces hint reliance as a measure of whether hinted success is likely to transfer back to the no-hint policy, and uses it to define a transfer-weighted reward that steers the hinter toward hints whose induced correct trajectories remain plausible without the hint.
Experiments across eight benchmarks and two backbone LLMs show that HiLL consistently outperforms GRPO and prior hint-based baselines.
Our analysis further reveals that transfer weighting keeps hint reliance low throughout training and naturally produces shorter, more conceptual hints.
Together, these results demonstrate the value of adaptive and transfer-aware hint learning for RL.

\bibliography{main}

@article{shao2024deepseekmath,
  title={Deepseekmath: Pushing the limits of mathematical reasoning in open language models},
  author={Shao, Zhihong and Wang, Peiyi and Zhu, Qihao and Xu, Runxin and Song, Junxiao and Bi, Xiao and Zhang, Haowei and Zhang, Mingchuan and Li, YK and Wu, Yang and others},
  journal={arXiv preprint arXiv:2402.03300},
  year={2024}
}

@misc{maa2024aime,
  author={{MAA Committees}},
  title={{AIME} Problems and Solutions},
  year={2024},
  url={https://artofproblemsolving.com/wiki/index.php/AIME_Problems_and_Solutions}
}

@misc{maa2025aime,
  author={{MAA Committees}},
  title={{AIME} Problems and Solutions},
  year={2025},
  url={https://artofproblemsolving.com/wiki/index.php/AIME_Problems_and_Solutions}
}

@article{wen2025rlvrboundary,
  title={Reinforcement learning with verifiable rewards implicitly incentivizes correct reasoning in base llms},
  author={Wen, Xumeng and Liu, Zihan and Zheng, Shun and Ye, Shengyu and Wu, Zhirong and Wang, Yang and Xu, Zhijian and Liang, Xiao and Li, Junjie and Miao, Ziming and others},
  journal={arXiv preprint arXiv:2506.14245},
  year={2025}
}

@article{guo2025deepseekr1,
  title={Deepseek-r1: Incentivizing reasoning capability in llms via reinforcement learning},
  author={Guo, Daya and Yang, Dejian and Zhang, Haowei and Song, Junxiao and Wang, Peiyi and Zhu, Qihao and Xu, Runxin and Zhang, Ruoyu and Ma, Shirong and Bi, Xiao and others},
  journal={arXiv preprint arXiv:2501.12948},
  year={2025}
}

@article{xiong2025reinforceada,
  title={Reinforce-ada: An adaptive sampling framework for reinforce-style llm training},
  author={Xiong, Wei and Ye, Chenlu and Liao, Baohao and Dong, Hanze and Xu, Xinxing and Monz, Christof and Bian, Jiang and Jiang, Nan and Zhang, Tong},
  journal={arXiv preprint arXiv:2510.04996},
  year={2025}
}

@article{chen2025nurl,
  title={Nudging the Boundaries of LLM Reasoning},
  author={Chen, Justin Chih-Yao and Peng, Becky Xiangyu and Choubey, Prafulla Kumar and Huang, Kung-Hsiang and Zhang, Jiaxin and Bansal, Mohit and Wu, Chien-Sheng},
  journal={arXiv preprint arXiv:2509.25666},
  year={2025}
}

@article{yao2025gvmraft,
  title={Optimizing Chain-of-Thought Reasoners via Gradient Variance Minimization in Rejection Sampling and RL},
  author={Yao, Jiarui and Hao, Yifan and Zhang, Hanning and Dong, Hanze and Xiong, Wei and Jiang, Nan and Zhang, Tong},
  journal={arXiv preprint arXiv:2505.02391},
  year={2025}
}

@article{nan2025ngrpo,
  title={Ngrpo: Negative-enhanced group relative policy optimization},
  author={Nan, Gongrui and Chen, Siye and Huang, Jing and Lu, Mengyu and Wang, Dexun and Xie, Chunmei and Xiong, Weiqi and Zeng, Xianzhou and Zhou, Qixuan and Li, Yadong and others},
  journal={arXiv preprint arXiv:2509.18851},
  year={2025}
}

@article{yu2025dapo,
  title={Dapo: An open-source llm reinforcement learning system at scale},
  author={Yu, Qiying and Zhang, Zheng and Zhu, Ruofei and Yuan, Yufeng and Zuo, Xiaochen and Yue, Yu and Dai, Weinan and Fan, Tiantian and Liu, Gaohong and Liu, Lingjun and others},
  journal={arXiv preprint arXiv:2503.14476},
  year={2025}
}

@article{zhang2025scaf,
  title={Scaf-GRPO: Scaffolded Group Relative Policy Optimization for Enhancing LLM Reasoning},
  author={Zhang, Xichen and Wu, Sitong and Zhu, Yinghao and Tan, Haoru and Yu, Shaozuo and He, Ziyi and Jia, Jiaya},
  journal={arXiv preprint arXiv:2510.19807},
  year={2025}
}

@article{zhang2025stephint,
  title={Stephint: Multi-level stepwise hints enhance reinforcement learning to reason},
  author={Zhang, Kaiyi and Lv, Ang and Li, Jinpeng and Wang, Yongbo and Wang, Feng and Hu, Haoyuan and Yan, Rui},
  journal={arXiv preprint arXiv:2507.02841},
  year={2025}
}

@misc{yang2024qwen25,
      title={Qwen2.5 Technical Report}, 
      author={An Yang and Baosong Yang and Beichen Zhang and Binyuan Hui and Bo Zheng and Bowen Yu and Chengyuan Li and Dayiheng Liu and Fei Huang and Haoran Wei and et al.},
      year={2025},
      eprint={2412.15115},
      archivePrefix={arXiv},
      primaryClass={cs.CL},
      url={https://arxiv.org/abs/2412.15115}, 
}

@article{meta2024llama3,
  title={Introducing meta llama 3: The most capable openly available llm to date},
  author={Meta, AI},
  journal={Meta AI},
  volume={2},
  number={5},
  pages={6},
  year={2024}
}

@inproceedings{rein2024gpqa,
  title={Gpqa: A graduate-level google-proof q\&a benchmark},
  author={Rein, David and Hou, Betty Li and Stickland, Asa Cooper and Petty, Jackson and Pang, Richard Yuanzhe and Dirani, Julien and Michael, Julian and Bowman, Samuel R},
  booktitle={First conference on language modeling},
  year={2024}
}

@inproceedings{moritz2018ray,
  title={Ray: A distributed framework for emerging $\{$AI$\}$ applications},
  author={Moritz, Philipp and Nishihara, Robert and Wang, Stephanie and Tumanov, Alexey and Liaw, Richard and Liang, Eric and Elibol, Melih and Yang, Zongheng and Paul, William and Jordan, Michael I and others},
  booktitle={13th USENIX symposium on operating systems design and implementation (OSDI 18)},
  pages={561--577},
  year={2018}
}

@inproceedings{kwon2023vllm,
  title={Efficient memory management for large language model serving with pagedattention},
  author={Kwon, Woosuk and Li, Zhuohan and Zhuang, Siyuan and Sheng, Ying and Zheng, Lianmin and Yu, Cody Hao and Gonzalez, Joseph and Zhang, Hao and Stoica, Ion},
  booktitle={Proceedings of the 29th symposium on operating systems principles},
  pages={611--626},
  year={2023}
}

@inproceedings{sheng2025verl,
  title={Hybridflow: A flexible and efficient rlhf framework},
  author={Sheng, Guangming and Zhang, Chi and Ye, Zilingfeng and Wu, Xibin and Zhang, Wang and Zhang, Ru and Peng, Yanghua and Lin, Haibin and Wu, Chuan},
  booktitle={Proceedings of the Twentieth European Conference on Computer Systems},
  pages={1279--1297},
  year={2025}
}

@article{wang2024mmlupro,
  title={Mmlu-pro: A more robust and challenging multi-task language understanding benchmark},
  author={Wang, Yubo and Ma, Xueguang and Zhang, Ge and Ni, Yuansheng and Chandra, Abhranil and Guo, Shiguang and Ren, Weiming and Arulraj, Aaran and He, Xuan and Jiang, Ziyan and others},
  journal={Advances in Neural Information Processing Systems},
  volume={37},
  pages={95266--95290},
  year={2024}
}

@inproceedings{he2024olympiadbench,
  title={Olympiadbench: A challenging benchmark for promoting agi with olympiad-level bilingual multimodal scientific problems},
  author={He, Chaoqun and Luo, Renjie and Bai, Yuzhuo and Hu, Shengding and Thai, Zhen and Shen, Junhao and Hu, Jinyi and Han, Xu and Huang, Yujie and Zhang, Yuxiang and others},
  booktitle={Proceedings of the 62nd Annual Meeting of the Association for Computational Linguistics (Volume 1: Long Papers)},
  pages={3828--3850},
  year={2024}
}

@article{lewkowycz2022minerva,
  title={Solving quantitative reasoning problems with language models},
  author={Lewkowycz, Aitor and Andreassen, Anders and Dohan, David and Dyer, Ethan and Michalewski, Henryk and Ramasesh, Vinay and Slone, Ambrose and Anil, Cem and Schlag, Imanol and Gutman-Solo, Theo and others},
  journal={Advances in neural information processing systems},
  volume={35},
  pages={3843--3857},
  year={2022}
}

@article{hendrycks2021math,
  title={Measuring mathematical problem solving with the math dataset},
  author={Hendrycks, Dan and Burns, Collin and Kadavath, Saurav and Arora, Akul and Basart, Steven and Tang, Eric and Song, Dawn and Steinhardt, Jacob},
  journal={arXiv preprint arXiv:2103.03874},
  year={2021}
}

@article{li2024numinamath,
  title={Numinamath: The largest public dataset in ai4maths with 860k pairs of competition math problems and solutions},
  author={Li, Jia and Beeching, Edward and Tunstall, Lewis and Lipkin, Ben and Soletskyi, Roman and Huang, Shengyi and Rasul, Kashif and Yu, Longhui and Jiang, Albert Q and Shen, Ziju and others},
  journal={Hugging Face repository},
  volume={13},
  number={9},
  pages={9},
  year={2024}
}

@article{huggingface2025openr1,
  title={Open r1: A fully open reproduction of deepseek-r1, January 2025},
  author={Face, Hugging},
  journal={URL https://github. com/huggingface/open-r1},
  volume={7},
  year={2025}
}

@article{yang2025qwen3,
  title={Qwen3 technical report},
  author={Yang, An and Li, Anfeng and Yang, Baosong and Zhang, Beichen and Hui, Binyuan and Zheng, Bo and Yu, Bowen and Gao, Chang and Huang, Chengen and Lv, Chenxu and others},
  journal={arXiv preprint arXiv:2505.09388},
  year={2025}
}

@article{zhang2025clpo,
  title={CLPO: Curriculum Learning meets Policy Optimization for LLM Reasoning},
  author={Zhang, Shijie and Sun, Guohao and Zhang, Kevin and Guo, Xiang and Guo, Rujun},
  journal={arXiv preprint arXiv:2509.25004},
  year={2025}
}

@article{shenfeld2026self,
  title={Self-Distillation Enables Continual Learning},
  author={Shenfeld, Idan and Damani, Mehul and H{\"u}botter, Jonas and Agrawal, Pulkit},
  journal={arXiv preprint arXiv:2601.19897},
  year={2026}
}

@article{hubotter2026sdpo,
  title={Reinforcement Learning via Self-Distillation},
  author={H{\"u}botter, Jonas and L{\"u}beck, Frederike and Behric, Lejs and Baumann, Anton and Bagatella, Marco and Marta, Daniel and Hakimi, Ido and Shenfeld, Idan and Buening, Thomas Kleine and Guestrin, Carlos and others},
  journal={arXiv preprint arXiv:2601.20802},
  year={2026}
}

@article{zhao2026opsd,
  title={Self-Distilled Reasoner: On-Policy Self-Distillation for Large Language Models},
  author={Zhao, Siyan and Xie, Zhihui and Liu, Mengchen and Huang, Jing and Pang, Guan and Chen, Feiyu and Grover, Aditya},
  journal={arXiv preprint arXiv:2601.18734},
  year={2026}
}

@inproceedings{deng2026hipo,
  title={HiPO: Self-Hint Policy Optimization for RLVR},
  author={Qiyuan, Deng and Chen, Kehai and Zhang, Min and Xu, Zhongwen},
  booktitle={The Fourteenth International Conference on Learning Representations}
}

@article{yan2025luffy,
  title={Learning to reason under off-policy guidance},
  author={Yan, Jianhao and Li, Yafu and Hu, Zican and Wang, Zhi and Cui, Ganqu and Qu, Xiaoye and Cheng, Yu and Zhang, Yue},
  journal={arXiv preprint arXiv:2504.14945},
  year={2025}
}

@article{parashar2025e2h,
  title={Curriculum reinforcement learning from easy to hard tasks improves LLM reasoning},
  author={Parashar, Shubham and Gui, Shurui and Li, Xiner and Ling, Hongyi and Vemuri, Sushil and Olson, Blake and Li, Eric and Zhang, Yu and Caverlee, James and Kalathil, Dileep and others},
  journal={arXiv preprint arXiv:2506.06632},
  year={2025}
}

@article{li2025knapsack,
  title={Knapsack rl: Unlocking exploration of llms via optimizing budget allocation},
  author={Li, Ziniu and Chen, Congliang and Yang, Tianyun and Ding, Tian and Sun, Ruoyu and Zhang, Ge and Huang, Wenhao and Luo, Zhi-Quan},
  journal={arXiv preprint arXiv:2509.25849},
  year={2025}
}

@article{xu2025pods,
  title={Not all rollouts are useful: Down-sampling rollouts in llm reinforcement learning},
  author={Xu, Yixuan Even and Savani, Yash and Fang, Fei and Kolter, J Zico},
  journal={arXiv preprint arXiv:2504.13818},
  year={2025}
}

@article{zheng2025greso,
  title={Act only when it pays: Efficient reinforcement learning for llm reasoning via selective rollouts},
  author={Zheng, Haizhong and Zhou, Yang and Bartoldson, Brian R and Kailkhura, Bhavya and Lai, Fan and Zhao, Jiawei and Chen, Beidi},
  journal={arXiv preprint arXiv:2506.02177},
  year={2025}
}

@article{qu2026pope,
  title={POPE: Learning to Reason on Hard Problems via Privileged On-Policy Exploration},
  author={Qu, Yuxiao and Setlur, Amrith and Smith, Virginia and Salakhutdinov, Ruslan and Kumar, Aviral},
  journal={arXiv preprint arXiv:2601.18779},
  year={2026}
}

@article{liao2026sage,
  title={Self-Hinting Language Models Enhance Reinforcement Learning},
  author={Liao, Baohao and Dong, Hanze and Xu, Xinxing and Monz, Christof and Bian, Jiang},
  journal={arXiv preprint arXiv:2602.03143},
  year={2026}
}

@inproceedings{
le2025zvp,
title={No Prompt Left Behind: Exploiting Zero-Variance Prompts in {LLM} Reinforcement Learning via Entropy-Guided Advantage Shaping},
author={Thanh-Long V. Le and Myeongho Jeon and Kim Vu and Viet Dac Lai and Eunho Yang},
booktitle={The Fourteenth International Conference on Learning Representations},
year={2026},
url={https://openreview.net/forum?id=kiXFIESZKv}
}

@article{mroueh2025grpoeffective,
  title={Reinforcement Learning with Verifiable Rewards: GRPO's Effective Loss, Dynamics, and Success Amplification},
  author={Mroueh, Youssef},
  journal={arXiv preprint arXiv:2503.06639},
  year={2025}
}
\bibliographystyle{plainnat}


\appendix

\section{Proof of Proposition~\ref{prop:transfer}}
\label{app:proof_transfer}

We restate Proposition~\ref{prop:transfer} for convenience.
\textit{For a question-hint pair $(q,h)$, let $P_h(\tau)=\pi_\theta(\tau\mid q{+}h)$ and $P(\tau)=\pi_\theta(\tau\mid q)$ denote the trajectory distributions under the hinted and original inputs, with success probabilities $p_h=P_h\bigl(r(\tau){=}1\bigr)$ and $p=P\bigl(r(\tau){=}1\bigr)$.
If $p_h>0$, then}
\[
\rho_c(q,h)
\;=\;
\log\frac{p_h}{p}
\;+\;
D_{\mathrm{KL}}\!\bigl(
    P_h(\cdot\mid r{=}1)
    \;\big\|\;
    P(\cdot\mid r{=}1)
\bigr),
\]
\textit{and therefore}
\[
p
\;\ge\;
p_h \cdot \exp\bigl(-\rho_c(q,h)\bigr).
\]

\begin{proof}
Let $\mathcal{S}=\{\tau:r(\tau)=1\}$ denote the set of correct trajectories.
For any $\tau\in\mathcal{S}$, write
\[
P_h(\tau)
\;=\;
p_h\;P_h(\tau\mid r{=}1),
\qquad
P(\tau)
\;=\;
p\;P(\tau\mid r{=}1).
\]
Therefore, the per-trajectory hint reliance (Eq.~\ref{eq:hint_reliance}) satisfies
\begin{align*}
    \rho(\tau;\,q,h)
\;&=\;
\log\frac{P_h(\tau)}{P(\tau)}\\
\;&=\;
\log\frac{p_h}{p}
\;+\;
\log\frac{P_h(\tau\mid r{=}1)}
         {P(\tau\mid r{=}1)}.
\end{align*}
Taking expectation under $P_h(\cdot\mid r{=}1)$ gives
\begin{align*}
    \rho_c(q,h)
\;&=\;
\log\frac{p_h}{p}
\;+\;
\mathbb{E}_{\tau\sim P_h(\cdot\mid r{=}1)}
\!\left[
    \log\frac{P_h(\tau\mid r{=}1)}
             {P(\tau\mid r{=}1)}
\right]\\
\;&=\;
\log\frac{p_h}{p}
\;+\;
D_{\mathrm{KL}}\!\bigl(
    P_h(\cdot\mid r{=}1)
    \;\big\|\;
    P(\cdot\mid r{=}1)
\bigr),
\end{align*}
which proves the identity.
Since
$D_{\mathrm{KL}}\!\bigl(
    P_h(\cdot\mid r{=}1)
    \,\big\|\,
    P(\cdot\mid r{=}1)
\bigr)\ge 0$,
we obtain
$\rho_c(q,h)\ge\log(p_h/p)$,
and rearranging yields
\[
p
\;\ge\;
p_h\cdot\exp\bigl(-\rho_c(q,h)\bigr).
\]
\end{proof}

\section{Prompts for Hinter and Reasoner}
\label{app:hint_generator_prompt}

This section lists all prompt templates used in HiLL. We do not set a system prompt for the hinter.
Note that for the hinted reasoner input, we simply append the hint after the original question without adding explicit bridging prose such as ``Here is a hint to help you:'' as used in \citet{liao2026sage}. We empirically observed that such phrasing increases hint reliance by inducing reasoner outputs like ``Given the hint, \ldots'' which are less likely under the no-hint input and thus less transferable.

\begin{promptbox}{User Prompt for Hint Generation}
You are a Pedagogical Hint Generator for a Mathematical Reasoner.

\bigskip
{TASK:} GENERATE A MINIMAL HINT

\bigskip
{OBJECTIVE}

The Reasoner failed the Original Question. Your goal is to generate a
minimal but useful hint that helps the Reasoner find the correct logical
path. Provide a strategic nudge, not a solution. The Reasoner must still
perform the deduction themselves.

\bigskip
{INPUTS}

1) Original Question: 

\texttt{\{original\_question\}}

\bigskip
2) Reasoner's Failed Attempt: 

\texttt{\{rollout\_responses\}}

\bigskip
3) Ground Truth Solution (Hidden from Reasoner; For Your Reference Only):

\texttt{\{ground\_truth\_solution\}}

\bigskip
{GUIDELINES}

1. Analyze the failure to identify the missing insight or misconception.

2. Provide a conceptual pointer, such as a relevant theorem, a structural
property, an alternative representation, or an intermediate goal.

3. Do not reveal the final answer or any specific numerical values,
formulas with substituted numbers, or computed intermediate steps.

\bigskip
{OUTPUT FORMAT (Follow strictly)}

<analysis>

[Briefly identify the logic breakdown and the necessary insight.]

</analysis>

\bigskip

<hint>

[Your concise hint --- 1 to 3 sentences maximum.]

</hint>
\end{promptbox}

\begin{promptbox}{System Prompt for Reasoner}
Please reason step by step, and put your final answer within \textbackslash boxed\{\}.
\end{promptbox}

\begin{promptbox}{User Prompt for Reasoner}
\texttt{\{question\}}
\end{promptbox}

\begin{promptbox}{User Prompt for Reasoner with Hint}
\texttt{\{question\}}

\bigskip
\texttt{\{hint\}}
\end{promptbox}

\end{document}